\documentclass[english]{article}
\usepackage[T1]{fontenc}
\usepackage[latin9]{inputenc}
\usepackage{float}
\usepackage{wrapfig}
\usepackage{units}
\usepackage{amsthm}
\usepackage{amsmath}
\usepackage{amssymb}
\usepackage{graphicx}
\usepackage{multirow}
\usepackage{enumerate}
\usepackage[centerlast]{caption}
\usepackage{subfigure}
\makeatletter

\newcommand{\lyxmathsym}[1]{\ifmmode\begingroup\def\b@ld{bold}
  \text{\ifx\math@version\b@ld\bfseries\fi#1}\endgroup\else#1\fi}

\newcommand{\shortVersion}[1]{#1} 

\providecommand{\tabularnewline}{\\}
\floatstyle{ruled}
\newfloat{algorithm}{tbp}{loa}
\providecommand{\algorithmname}{Algorithm}
\floatname{algorithm}{\protect\algorithmname}

\usepackage{nips13submit_e,times,hyperref,url}
\nipsfinalcopy
\usepackage[T1]{fontenc}
\usepackage{lmodern}

\newtheorem{theorem}{Theorem}
\newtheorem{lemma}{Lemma}
\newtheorem{definition}{Definition}
\newtheorem{assumption}{Assumption}

\makeatother

\usepackage{babel}
\begin{document}
\global\long\def\phimin{\phi_{\mbox{min}}}

\global\long\def\sigmamax{\sigma_{\mbox{max}}}

\global\long\def\reals{\mathbb{R}}

\global\long\def\norm#1{\left\Vert #1\right\Vert }

\global\long\def\alg{\mathrm{L}}

\global\long\def\ind{\mathbb{I}}

\title{Using multiple samples to learn mixture models}

\author{
Jason Lee\thanks{Work done while the author was an intern at Microsoft Resaerch}\\
  Stanford University\\
  Stanford, USA\\
  \texttt{jdl17@stanford.edu}\\
\And
Ran Gilad-Bachrach\\
  Microsoft Research\\
  Redmond, USA\\
  \texttt{rang@microsoft.com}\\
\And
Rich Caruana\\
  Microsoft Research\\
  Redmond, USA\\
  \texttt{rcaruana@microsoft.com}
}

\maketitle
\begin{abstract}
In the mixture models problem it is assumed that there are $K$
distributions $\theta_{1},\ldots,\theta_{K}$ and one gets to observe
a sample from a mixture of these distributions with unknown coefficients.
The goal is to associate instances with their generating
distributions, or to identify the parameters of the hidden distributions.
In this work we make the assumption that we have access to several
samples drawn from the same $K$ underlying distributions, but with 
different mixing weights. As with topic modeling, having multiple samples is often a reasonable 
assumption.  Instead of pooling the data into one sample, we prove that 
it is possible to use the differences between the samples to better recover 
the underlying structure. We present algorithms that 
recover the underlying structure 
under milder assumptions than the current state of art when either the 
dimensionality or the separation is high.  The methods, when applied to 
topic modeling, allow generalization to words not present in the training 
data.
\end{abstract}

\section{Introduction}

The mixture model has been studied extensively from several directions.
In one setting it is assumed that there is a single sample, that is a single collection of instances, from which
one has to recover the hidden information. A line of studies on clustering theory, starting
from \cite{Dasgupta99} has proposed to address this problem by finding
a projection to a low dimensional space and solving the problem in
this space. The goal of this projection is to reduce the dimension
while preserving the distances, as much as possible, between the means
of the underlying distributions. We will refer to this line as MM (Mixture Models).
On the other end of the spectrum, Topic modeling (TM), \cite{PTRV98, BleiNgJordan03}, assumes multiple samples (documents) 
that are mixtures, with different weights of the underlying distributions (topics) over words.

Comparing the two lines presented above shows some similarities and
some differences. Both models assume the same generative structure: a point (word)
is generated by first choosing the distribution $\theta_i$ using the mixing weights
and then selecting a point (word) according to this distribution. The 
goal of both models is to recover information about the generative model
 (see \cite{AroraKannan01} for more on that). However,
there are some key differences:
\begin{enumerate}[(a)]
\item 
In MM, there exists a single sample
to learn from. In TM, each document is a mixture of the topics,
but with different mixture weights. 
\item
In MM, the points are represented as feature vectors
while in TM the data is represented as a word-document co-occurrence
matrix. As a consequence, the model generated by TM cannot assign
words that did not previously appear in any document to topics. 
\item
TM assumes high density of the samples, i.e., that the each word appears multiple times. However, if the topics were not discrete distributions, as is mostly the case in MM, each "word" (i.e., value) would typically appear either zero or one time, which makes the co-occurrence matrix useless.
\end{enumerate}
In this work we try to close the gap between MM and TM. Similar to
TM, we assume that multiple samples are available. However,
we assume that points (words) are presented as feature vectors and
the hidden distributions may be continuous. This allows us to solve
problems that are typically hard in the MM model with greater ease
and generate models that generalize to points not
in the training data which is something that TM cannot do.

\subsection{Definitions and Notations}
We assume a mixture model in which there are $K$ mixture components
$\theta_{1},\ldots,\theta_{K}$ defined over the space $X$. These
mixture components are probability measures over $X$. We assume that
there are $M$ mixture models (samples), each drawn with different mixture weights
$\Phi^{1},\ldots,\Phi^{M}$ such that $\Phi^{j}=(\phi_{1}^{j},\ldots,\phi_{K}^{j})$
where all the weights are non-negative and sum to $1$. Therefore,
we have $M$ different probability measures $D_{1},\ldots,D_{M}$
defined over $X$ such that for a measurable set $A$ and $j=1,\ldots,M$
we have $D_{j}(A)=\sum_{i}\phi_{i}^{j}\theta_{i}\left(A\right)$.
We will denote by $\phimin$ the minimal value of $\phi_{i}^{j}$.

In the first part of this work, we will provide an algorithm that given samples $S_1,\ldots,S_M$ from the mixtures $D_1,\ldots,D_M$ finds a low-dimensional embedding that preserves the distances between the means of each mixture.

In the second part of this work we will assume that the mixture components
have disjoint supports. Hence we will assume that $X=\cup_{j}C_{j}$
such that the $C_{j}$'s are disjoint and for every $j,\,\,\,\theta_{j}(C_{j})=1$. Given samples $S_1,\ldots,S_M$, we will provide an algorithm that finds the supports of the underlying distributions, and thus clusters the samples.

\subsection{Examples}
Before we dive further in the discussion of our methods and how they
compare to prior art, we would like to point out that the model we
assume is realistic in many cases. Consider the
following example: assume that one would like to cluster medical records
to identify sub-types of diseases (e.g., different types of heart disease). In the classical clustering setting
(MM), one would take a sample of patients and try to divide them based
on some similarity criteria into groups. However, in many cases, one
has access to data from different hospitals in different geographical
locations. The communities being served by the different hospitals
may be different in socioeconomic status, demographics, genetic backgrounds, and exposure to climate and environmental hazards. Therefore, different disease sub-types are likely to appear in
different ratios in the different hospital. However, if patients in
two hospitals acquired the same sub-type of a disease, parts of their medical
records will be similar.

Another example is object classification in images. Given an image,
one may break it to patches, say of size 10x10 pixels. These patches
may have different distributions based on the object in that part of the image.
Therefore, patches from images taken at different locations will have
different representation of the underlying distributions. Moreover,
patches from the center of the frame are more likely to contain
parts of faces than patches from the perimeter of the picture. At
the same time, patches from the bottom of the picture are more likely
to be of grass than patches from the top of the picture.

In the first part of this work we discuss the problem of identifying
the mixture component from multiple samples when the means of the
different components differ and variances are bounded.
We focus on the problem of finding a low dimensional embedding of the
data that preserves the distances between the means since the problem
of finding the mixtures in a low dimensional space has already been
address (see, for example \cite{AroraKannan01}). 
Next, we address a different case in which we assume
that the support of the hidden distributions is disjoint. We show
that in this case we can identify the supports of each distribution.
Finally we demonstrate our approaches on toy problems. 
The proofs of the theorems and lemmas appear in the appendix.
Table~\ref{tab:scenarios}
summarizes the applicability of the algorithms presented here to the
different scenarios. 

\begin{wraptable}{o}{0.5\columnwidth}%
%
\begin{centering}
\begin{tabular}{|c|c|c|}
\hline 
 & \textbf{Disjoint} & \textbf{Overlapping}\tabularnewline
 & \textbf{clusters} & \textbf{clusters} \tabularnewline
\hline 
\hline 
\textbf{High} & \multirow{2}{*}{DSC, MSP} & \multirow{2}{*}{MSP}\tabularnewline
\textbf{dimension} & & \tabularnewline
\hline 
\textbf{Low} & \multirow{2}{*}{DSC} & \tabularnewline
\textbf{dimension} & &\tabularnewline
\hline 
\end{tabular}
\par\end{centering}

\caption{\label{tab:scenarios} Summary of the scenarios the MSP (Multi Sample Projection)
algorithm and the DSC (Double Sample Clustering) algorithm are designed to address.}
\end{wraptable}

\subsection{Comparison to prior art\label{sub:Comparison-to-prior}}

There are two common approaches in the theoretical study of the MM
model. The method of moments \cite{KalaiMoitraValiant10, MoitraValiant10, BelkinSinha10} 
allows the recovery of the model but requires exponential running
time and sample sizes. The other approach, to which we compare our results,
uses a two stage approach. In the first stage, the data is projected
to a low dimensional space and in the second stage the association
of points to clusters is recovered. Most of the results with this approach assume that
the mixture components are Gaussians. Dasgupta \cite{Dasgupta99}, in a seminal
paper, presented the first result in this line. He used random projections
to project the points to a space of a lower dimension. This work assumes
that separation is at least $\Omega(\sigmamax\sqrt{n})$. This result
has been improved in a series of papers. Arora and Kannan \cite{AroraKannan01} presented
algorithms for finding the mixture components which are, in most cases,
polynomial in $n$ and $K$. Vempala and Wang \cite{VempalaWang02} used PCA to reduce
the required separation to $\Omega\left(\sigmamax K^{\nicefrac{1}{4}}\log^{\nicefrac{1}{4}}\left(\nicefrac{n}{\phimin}\right)\right)$.
They use PCA to project on the first $K$ principal components, however,
they require the Gaussians to be spherical. Kanan, Salmasian and Vempala \cite{KannanSalmasianVempala05} 
used similar spectral methods but were able to improve the results
to require separation of only $c\sigmamax\nicefrac{K^{\nicefrac{2}{3}}}{\phimin^{2}}$.
Chaudhuri \cite{ChaudhuriRao08} have suggested using correlations and independence
between features under the assumption that the means of the Gaussians differ
on many features. They require separation of $\Omega\left(\sigmamax\sqrt{K\log(K\sigmamax\log n/\phimin)}\right)$,
however they assume that the Gaussians are axis aligned and that the
distance between the centers of the Gaussians is spread across $\Omega\left(K\sigmamax\log n/\phimin\right)$
coordinates.

We present a method to project the problem into a space of dimension
$d^{*}$ which is the dimension of the affine space spanned by the
means of the distributions. We can find this projection and maintain
the distances between the means to within a factor of $1-\epsilon$.
The different factors, $\sigmamax,\, n$ and $\epsilon$ will affect
the sample size needed, but do not make the problem impossible. This
can be used as a preprocessing step for any of the results discussed above. For example, combining
with \cite{Dasgupta99} yields an algorithm that requires a separation
of only $\Omega\left(\sigmamax\sqrt{d^{*}}\right)\leq\Omega\left(\sigmamax\sqrt{K}\right)$.
However, using \cite{VempalaWang02} will result in separation requirement
of $\Omega\left(\sigmamax\sqrt{K\log\left(K\sigmamax\log d^{*}/\phimin\right)}\right)$.
There is also an improvement in terms of the value of $\sigmamax$
since we need only to control the variance in the affine space spanned
by the means of the Gaussians and do not need to restrict the variance
in orthogonal directions, as long as it is finite. Later we also show
that we can work in a more generic setting where the distributions
are not restricted to be Gaussians as long as the supports of the
distributions are disjoint. 
While the disjoint assumption may seem too strict, we note that the results presented 
above make very similar assumptions.  For example, even
if the required separation is $\sigmamax K^{\nicefrac{1}{2}}$ then
if we look at the Voronoi tessellation around the centers of the Gaussians,
each cell will contain at least $1-\left(2\pi\right)^{-1}K^{\nicefrac{3}{4}}\exp\left(-\nicefrac{K}{2}\right)$ of the mass 
of the Gaussian.
Therefore, when $K$ is large, the supports of the Gaussians are almost
disjoint.

\section{Projection for overlapping components}

In this section we present a method to use multiple samples to
project high dimensional mixtures to a low dimensional space while
keeping the means of the mixture components well separated. The main
idea behind the Multi Sample Projection (MSP) algorithm is simple.
Let $\mu_{i}$ be the mean of the $i$'th component $\theta_{i}$
and let $E_{j}$ be the mean of the $j$'th mixture $D_{j}$. From
the nature of the mixture, $E_{j}$ is in the convex-hull of $\mu_{1},\ldots,\mu_{K}$
and hence in the affine space spanned by them; this is demonstrated
in Figure~\ref{fig:MSP demo}. Under mild assumptions, if we have
sufficiently many mixtures, their means will span the affine space
spanned by $\mu_{1},\ldots,\mu_{K}$. Therefore, the MSP algorithm
estimates the $E_{j}$'s and projects to the affine space they span.
The reason for selecting this sub-space is that by projecting on this
space we maintain the distance between the means while reducing
the dimension to at most $K-1$. The MSP algorithm is presented in
Algorithm~\ref{alg:Multi-Sample-Projection-(MSP)}. In the following
theorem we prove the main properties of the MSP algorithm. We will assume that $X=\reals^{n}$, the first two moments of $\theta_{j}$
are finite, and $\sigmamax^{2}$ denotes maximal variance of any of
the components in any direction. The separation of the mixture components
is $\min_{j\neq j^{\prime}}\Vert\mu_{j}-\mu_{j^{\prime}}\Vert$. Finally,
we will denote by $d^{*}$ the dimension of the affine space spanned
by the $\mu_{j}$\textquoteright{}s. Hence, $d^{*}\leq K-1$.

\begin{algorithm}

\caption{\textbf{Multi Sample Projection (MSP) \label{alg:Multi-Sample-Projection-(MSP)}}}

\textbf{Inputs:}

Samples $S_{1},\ldots,S_{m}$ from mixtures $D_{1},\ldots,D_{m}$

\textbf{Outputs:}

Vectors $\bar{v}_{1},\ldots,\bar{v}_{m-1}$ which span the projected
space

\textbf{Algorithm:}
\begin{enumerate}
\item For $j=1,\ldots,m$ let $\bar{E}_{j}$ be the mean of the sample $S_{j}$
\item For $j=1,\ldots,m-1$ let $\bar{v}_{j}=\bar{E}_{j}-\bar{E}_{j+1}$
\item return $\bar{v}_{1},\ldots,\bar{v}_{m-1}$\end{enumerate}

\end{algorithm}

\begin{theorem} \textbf{MSP Analysis\label{thm: MSP-Analysis}}

Let $E_{j}=E\left[D_{j}\right]$ and let $v_{j}=E_{j}-E_{j+1}$. Let
$N_{j}=\left|S_{j}\right|$. The following holds for MSP:
\begin{enumerate}
\item The computational complexity of the MSP algorithm is $n\sum_{j=1}^{M}N_{j}+2n\left(m-1\right)$
where $n$ is the original dimension of the problem.
\item  For any $\epsilon>0$, $\Pr\left[\sup_{j}\norm{E_{j}-\bar{E}_{j}}>\epsilon\right]\leq\frac{n\sigmamax^{2}}{\epsilon^{2}}\sum_{j}\frac{1}{N_{j}}\,\,\,.$
\item Let $\bar{\mu}_{i}$ be the projection of $\mu_{i}$ on the space
spanned by $\bar{v}_{1},\ldots,\bar{v}_{M-1}$ and assume that $\forall i,\,\,\,\mu_{i}\in\mbox{span}\left\{ v_{j}\right\} $.
Let $\alpha_{j}^{i}$ be such that $\mu_{i}=\sum_{j}\alpha_{j}^{i}v_{j}$
and let $A=\max_{i}\sum\left|\alpha_{j}^{i}\right|$ then with probability
of at least $1-\frac{n\sigmamax^{2}}{\epsilon^{2}}\sum_{j}\frac{1}{N_{j}}$
\[
\Pr\left[\max_{i,i^{\prime}}\left|\norm{\mu_{i}-\mu_{i^{\prime}}}-\norm{\bar{\mu}_{i}-\bar{\mu}_{i^{\prime}}}\right|>\epsilon\right]\leq\frac{4n\sigmamax^{2}A^{2}}{\epsilon^{2}}\sum_{j}\frac{1}{N_{j}}\,\,\,.
\]

\end{enumerate}
\end{theorem}
\begin{wrapfigure}{o}{0.5\columnwidth}%
\begin{centering}
\includegraphics[width=0.4\textwidth]{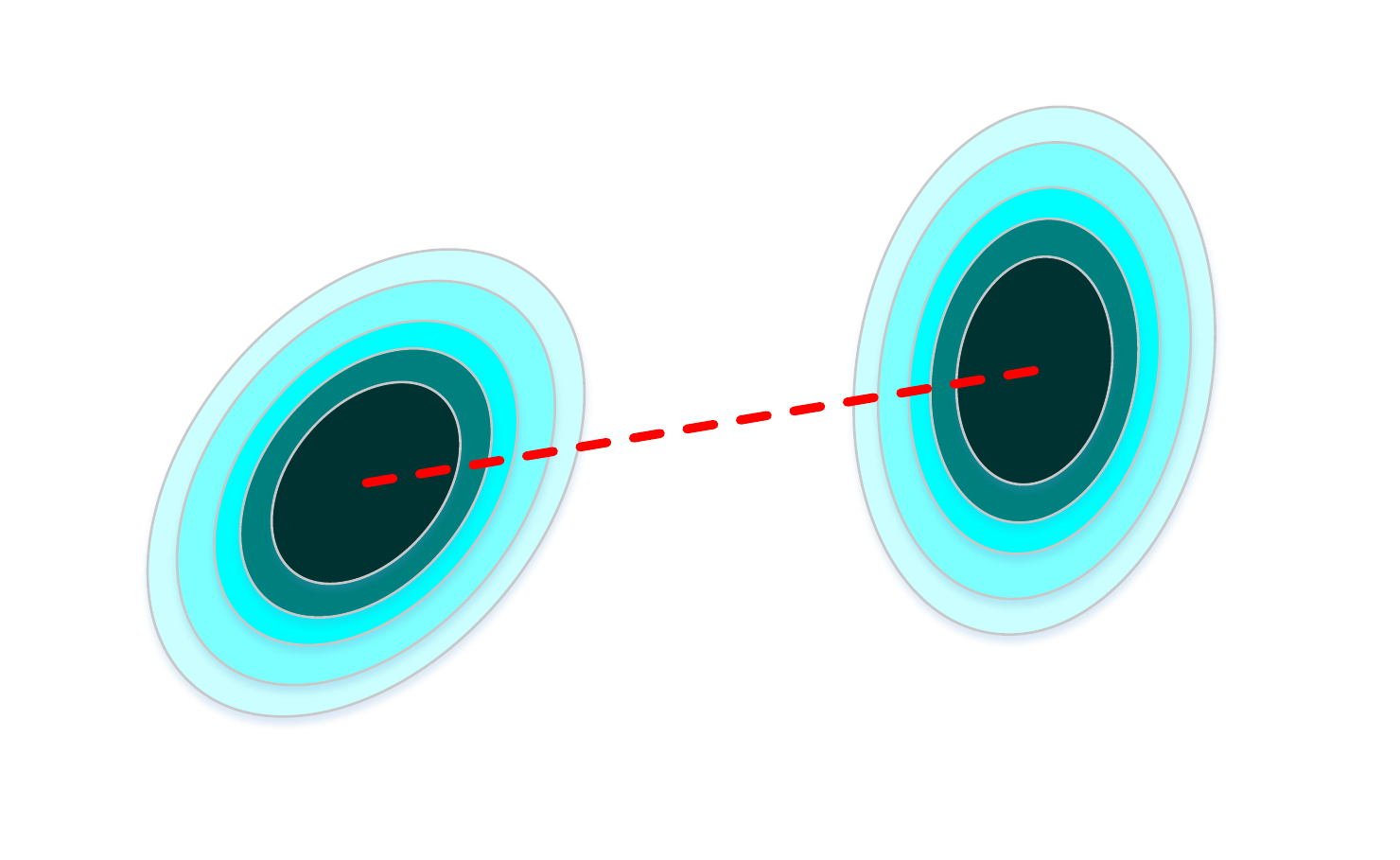}
\par\end{centering}

\caption{\label{fig:MSP demo}The mean of the mixture components will be in
the convex hull of their means demonstrated here by the red line.}
\end{wrapfigure}%

The MSP analysis theorem shows that with large enough samples, the
projection will maintain the separation between the centers of the
distributions. Moreover, since this is a projection, the variance
in any direction cannot increase. The value of $A$ measures the
complexity of the setting. If the mixing coefficients are very different
in the different samples then $A$ will be small. However, if the mixing
coefficients are very similar, a larger sample is required. Nevertheless,
the size of the sample needed is polynomial in the parameters of the
problem. It is also apparent that with large enough samples, a good
projection will be found, even with large variances, high dimensions
and close centroids. 

A nice property of the bounds presented here is that they assume only
bounded first and second moments. Once a projection to a low dimensional
space has been found, it is possible to find the clusters using approaches 
presented in section~\ref{sub:Comparison-to-prior}. However, the analysis of the
MSP algorithm assumes that the means of $E_{1},\ldots,E_{M}$ span
the affine space spanned by $\mu_{1},\ldots,\mu_{K}$. Clearly, this
implies that we require that $m>d^{*}$. However, when $m$ is
much larger than $d^{*}$, we might end-up with a projection on
too large a space. This could easily be fixed since in this case, $\bar{E}_{1},\ldots,\bar{E}_{m}$
will be almost co-planar in the sense that there will be an affine
space of dimension $d^{*}$ that is very close to all these points
and we can project onto this space.

\section{Disjoint supports and the Double Sample Clustering (DSC) algorithm}

In this section we discuss the case where the underlying
distributions have disjoint supports. In this case, we do not make
any assumption about the distributions. For example, we do not
require finite moments. However, as in the mixture of Gaussians case
some sort of separation between the distributions is needed, this
is the role of the disjoint supports. 

We will show that given two samples from mixtures with
different mixture coefficients, it is possible to find the supports
of the underlying distributions (clusters) by building a tree of classifiers
such that each leaf represents a cluster. The tree is constructed
in a greedy fashion. First we take the two samples, from the two distributions,
and reweigh the examples such that the two samples will have the
same cumulative weight. Next, we train a classifier to separate between
the two samples. This classifier becomes the root of the tree. It
also splits each of the samples into two sets. We take all the examples
that the classifier assign to the label $+1(-1)$, reweigh them and
train another classifier to separate between the two samples. We keep
going in the same fashion until we can no longer find a classifier
that splits the data significantly better than random. 

To understand why this algorithm works it is easier to look first
at the case where the mixture distributions are known. If $D_{1}$
and $D_{2}$ are known, we can define the $L_{1}$ distance between
them as $L_{1}\left(D_{1},D_{2}\right)=\sup_{A}\left|D_{1}\left(A\right)\lyxmathsym{\textendash}D_{2}\left(A\right)\right|$.%
\footnote{the supremum is over all the measurable sets.%
} It turns out that the supremum is attained by a set $A$ such that
for any $i$, $\mu_{i}\left(A\right)$ is either zero or one. Therefore,
any inner node in the tree splits the region without breaking clusters.
This process proceeds until all the points associated with a leaf
are from the same cluster in which case, no classifier can distinguish
between the classes.

When working with samples, we have to tolerate some error and prevent
overfitting. One way to see that is to look at the problem of approximating
the $L_{1}$ distance between $D_{1}$ and $D_{2}$ using samples
$S_{1}$ and $S_{2}$. One possible way to do that is to define $\hat{L}_{1}=\sup_{A}\left|\frac{A\cap S_{1}}{\bigl|S_{1}\bigr|}-\frac{A\cap S_{2}}{\bigl|S_{2}\bigr|}\right|$.
However, this estimate is almost surely going to be $1$ if the underlying
distributions are absolutely continuous. Therefore, one has to restrict the class
from which $A$ can be selected to a class of VC dimension small enough
compared to the sizes of the samples. We claim that asymptotically,
as the sizes of the samples increase, one can increase the complexity
of the class until the clusters can be separated. 

\begin{algorithm}
\caption{\textbf{Double Sample Clustering (DSC)}}

\textbf{Inputs:}
\begin{itemize}
\item Samples $S_{1},S_{2}$ 
\item A binary learning algorithm $\alg$ that given samples $S_{1},S_{2}$
with weights $w_{1},w_{2}$ finds a classifier $h$ and an estimator
$e$ of the error of $h$.
\item A threshold $\tau>0$.
\end{itemize}
\textbf{Outputs:}
\begin{itemize}
\item A tree of classifiers
\end{itemize}
\textbf{Algorithm:}
\begin{enumerate}
\item Let $w_{1}=1$ and $w_{2}=\nicefrac{\left|S_{1}\right|}{\left|S_{2}\right|}$
\item Apply $\alg$ to $S_{1}$ \& $S_{2}$ with weights $w_{1}$ \&
$w_{2}$ to get the classifier $h$ and estimator $e$.
\item If $e\geq\frac{1}{2}-\tau$, 

\begin{enumerate}
\item return a tree with a single leaf.
\end{enumerate}
\item else

\begin{enumerate}
\item For $j=1,2$, let $S_{j}^{+}=\left\{ x\in S_{j}\,\,\mbox{s.t.}\,\, h\left(x\right)>0\right\} $
\item For $j=1,2$, let $S_{j}^{-}=\left\{ x\in S_{j}\,\,\mbox{s.t.}\,\, h\left(x\right)<0\right\} $
\item Let $T^{+}$ be the tree returned by the DSC algorithm applied to
$S_{1}^{+}$ and $S_{2}^{+}$
\item Let $T^{-}$ be the tree returned by the DSC algorithm applied to
$S_{1}^{-}$ and $S_{2}^{-}$
\item return a tree in which $c$ is at the root node and $T^{-}$ is its
left subtree and $T^{+}$ is its right subtree\end{enumerate}
\end{enumerate}
\end{algorithm}

Before we proceed, we recall a result of \cite{BBCP07}
that shows the relation between classification and the $L_{1}$ distance.
We will abuse the notation and treat $A$ both as a subset
and as a classifier. If we mix $D_{1}$ and $D_{2}$ with equal weights
then
\begin{eqnarray*}
\mbox{err}\left(A\right) & = & D_{1}\left(X\setminus A\right)+D_{2}\left(A\right)\\
 & = & 1-D_{1}\left(A\right)+D_{2}\left(A\right)\\
 & = & 1-\left(D_{1}\left(A\right)-D_{2}\left(A\right)\right)\,\,\,.
\end{eqnarray*}
Therefore, minimizing the error is equivalent to maximizing the $L_{1}$
distance.

The key observation for the DSC algorithm is that if $\phi_{i}^{1}\neq\phi_{i}^{2}$,
then a set $A$ that maximizes the $L_{1}$ distance between $D_{1}$
and $D_{2}$ is aligned with cluster boundaries (up to a measure zero).
\begin{wrapfigure}{}{0.5\columnwidth}%
\vspace*{-0.5in}
\begin{centering}
\includegraphics[width=0.45\textwidth]{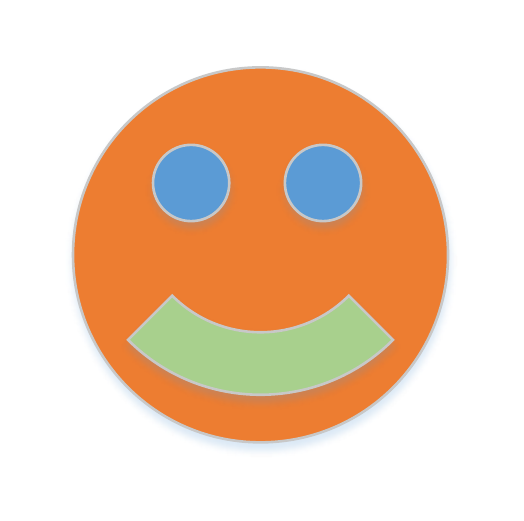}
\par\end{centering}
\vspace*{-0.3in}
\caption{\label{fig:DSC demo}Demonstration of the DSC algorithm. Assume
that $\Phi^{1}=\left(0.4,0.3,0.3\right)$ for the orange,
green and blue regions respectively and $\Phi^{2}=\left(0.5,0.1,0.4\right)$. The green region maximizes the $L_{1}$ distance
and therefore will be separated from the blue and orange. Conditioned
on these two regions, the mixture coefficients are  $\Phi^{1}_{\mbox{orange, blue}} = \left(\nicefrac{4}{7},\nicefrac{3}{7}\right)$
and $\Phi^{2}_{\mbox{orange, blue}} = \left(\nicefrac{5}{9},\nicefrac{4}{9}\right)$. The region that maximized
this conditional $L_{1}$ is the orange regions that will be separated
from the blue.}
\vspace*{-0.7in}
\end{wrapfigure}
Furthermore, $A$ contains all the clusters for which $\phi_{i}^{1}>\phi_{i}^{2}$
and does not contain all the clusters for which $\phi_{i}^{1}<\phi_{i}^{2}$.
Hence, if we split the space to $A$ and $\bar{A}$ we have few clusters
in each side. By applying the same trick recursively in each side we
keep on bisecting the space according to cluster boundaries until
subspaces that contain only a single cluster remain. These sub-spaces
cannot be further separated and hence the algorithm will stop. Figure~\ref{fig:DSC demo}
demonstrates this idea. The following lemma states this argument mathematically:

\begin{lemma} \label{lemma: DSC full knowledge} If $D_{j}=\sum_{i}\phi_{i}^{j}\theta_{i}$
then
\begin{enumerate}
\item $L_{1}\left(D_{1},D_{2}\right)\leq\sum_{i}\max\left(\phi_{i}^{1}-\phi_{i}^{2},0\right)$.
\item If $A^{*}=\cup_{i:\phi_{i}^{1}>\phi_{i}^{2}}C_{i}$ then $D_{1}\left(A^{*}\right)-D_{2}\left(A^{*}\right)=\sum_{i}\max\left(\phi_{i}^{1}-\phi_{i}^{2},0\right)$.
\item If $\forall i,\,\,\phi_{i}^{1}\neq\phi_{i}^{2}$ and $A$ is such
that $D_{1}\left(A\right)-D_{2}\left(A\right)=L_{1}\left(D_{1},D_{2}\right)$
then $\forall i,\,\,\,\theta_{i}\left(A\Delta A^{*}\right)=0$.
\end{enumerate}
\end{lemma}

We conclude from Lemma~\ref{lemma: DSC full knowledge} that if
$D_{1}$ and $D_{2}$ were explicitly known and one could have found
a classifier that best separates between the distributions,
that classifier would not break clusters as long as the mixing coefficients
are not identical. In order for this to hold when the separation
is applied recursively in the DSC algorithm it suffices to have that
for every $I\subseteq\left[1,\ldots,K\right]$ if $\left|I\right|>1$
and $i\in I$ then 
\[
\frac{\phi_{i}^{1}}{\sum_{i^{\prime}\in I}\phi_{i^{\prime}}^{1}}\neq\frac{\phi_{i}^{2}}{\sum_{i^{\prime}\in I}\phi_{i^{\prime}}^{2}}
\]
to guarantee that at any stage of the algorithm clusters will not
be split by the classifier (but may be sections of measure zero). This
is also sufficient to guarantee that the leaves will contain single
clusters.

In the case where data is provided through a finite sample then some
book-keeping is needed. However, the analysis follows the same path.
We show that with samples large enough, clusters are only minimally
broken. For this to hold we require that the learning algorithm $\alg$
separates the clusters according to this definition:

\begin{definition} For $I\subseteq\left[1,\ldots,K\right]$ let $c_I : X \mapsto \{\pm 1\}$ be such that $c_I(x) =1$ if $x\in\cup_{i\in I}C_i$ and $c_I(x)=-1$ otherwise.
A learning algorithm $\alg$ \emph{separates} $C_{1},\ldots,C_{K}$
if for every $\epsilon,\delta>0$ there exists $N$ such that for
every $n>N$ and every measure $\nu$ over $X\times\left\{ \pm1\right\}$ with probability $1-\delta$ over samples from $\nu^n$:
\begin{enumerate}
\item The algorithm $L$ returns an hypothesis $h:X\mapsto\left\{ \pm1\right\} $
and an error estimator $e\in\left[0,1\right]$ such that $\left|\Pr_{x,y\sim\nu}\left[h\left(x\right)\neq y\right]-e\right|\leq\epsilon$
\item $h$ is such that 
\[
\forall I,\,\,\,\,\Pr_{x,y\sim\nu}\left[h\left(x\right)\neq y\right]<\Pr_{x,y\sim\nu}\left[c_{I}\left(x\right)\neq y\right]+\epsilon\,\,\,.
\]

\end{enumerate}
\end{definition}

Before we introduce the main statement, we define what it means for
a tree to cluster the mixture components:

\begin{definition} A \emph{clustering tree} is a tree in which in
each internal node is a classifier and the points that end in a certain
leaf are considered a cluster. A\emph{ }clustering tree \emph{$\epsilon$-clusters}
the mixture coefficient $\theta_{1},\ldots,\theta_{K}$ if for every
$i\in1,\ldots,K$ there exists a leaf in the tree such that the cluster
$L\subseteq X$ associated with this leaf is such that $\theta_{i}\left(L\right)\geq1-\epsilon$
and $\theta_{i^{\prime}}\left(L\right)<\epsilon$ for every $i^{\prime}\neq i$.

\end{definition}

To be able to find a clustering tree, the two mixtures have to be
different. The following definition captures the gap which is the
amount of difference between the mixtures.

\begin{definition} \label{def: gap}Let $\Phi^{1}$ and $\Phi^{2}$
be two mixture vectors. The \emph{gap}, $g$, between them is
\[
g=\min\left\{ \left|\frac{\phi_{i}^{1}}{\sum_{i^{\prime}\in I}\phi_{i^{\prime}}^{1}}-\frac{\phi_{i}^{2}}{\sum_{i^{\prime}\in I}\phi_{i^{\prime}}^{2}}\right|\,:\, I\subseteq\left[1,\ldots,K\right]\mbox{ and }\left|I\right|>1\right\} \,\,\,\,.
\]

We say that $\Phi$ is $b$ \emph{bounded away} from zero if 
$
b\leq\min_{i}\phi_{i}
$.

\end{definition}

\begin{theorem} \label{Thm: DSC} Assume that $\alg$ separates $\theta_{1},\ldots,\theta_{K}$,
there is a gap $g>0$ between $\Phi^{1}$ and $\Phi^{2}$ and both
$\Phi^{1}$ and $\Phi^{2}$ are bounded away from zero by $b>0$.
For every $\epsilon^{*},\delta^{*}>0$ there exists $N=N\left(\epsilon^{*},\delta^{*},g,b,K\right)$
such that given two random samples of sizes $N<n_{1},n_{2}$ from
the two mixtures, with probability of at least $1-\delta^{*}$ the
DSC algorithm will return a clustering tree which $\epsilon^{*}$-clusters
$\theta_{1},\ldots,\theta_{K}$ when applied with the threshold $\tau=\nicefrac{g}{8}$.
\end{theorem}

\section{Empirical evidence}

\begin{figure}
  \vspace*{-0.4in}
  \centering
  \begin{subfigure}[Accuracy with spherical Gaussians   \label{fig:spherical}
 ]
{
    \includegraphics[trim = 20mm 55mm 20mm 25mm,width=0.45\textwidth]{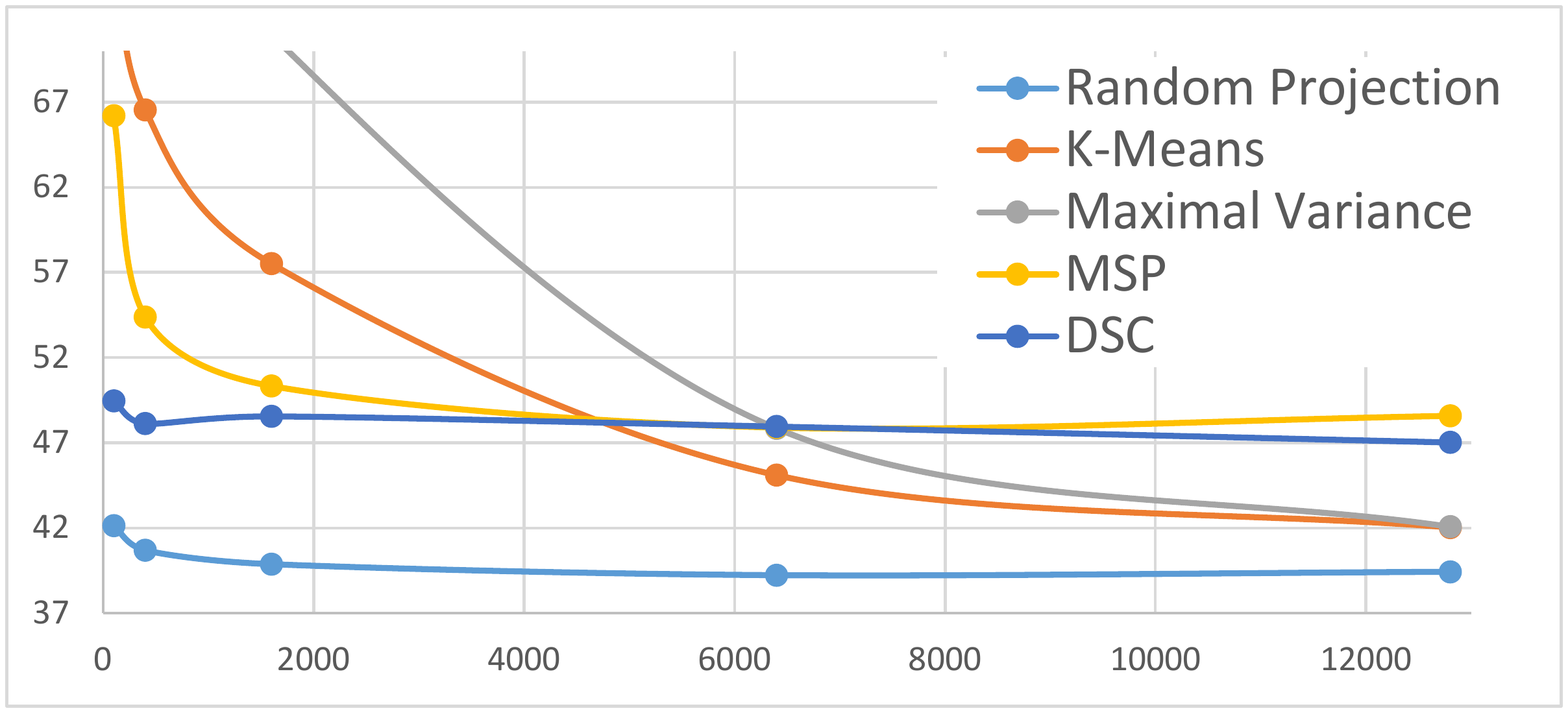}}
  \end{subfigure}
  \begin{subfigure}[Average accuracy with skewed Gaussians\label{fig:skewed}]{
    \includegraphics[trim = 20mm 55mm 20mm 25mm, width=0.45\textwidth]{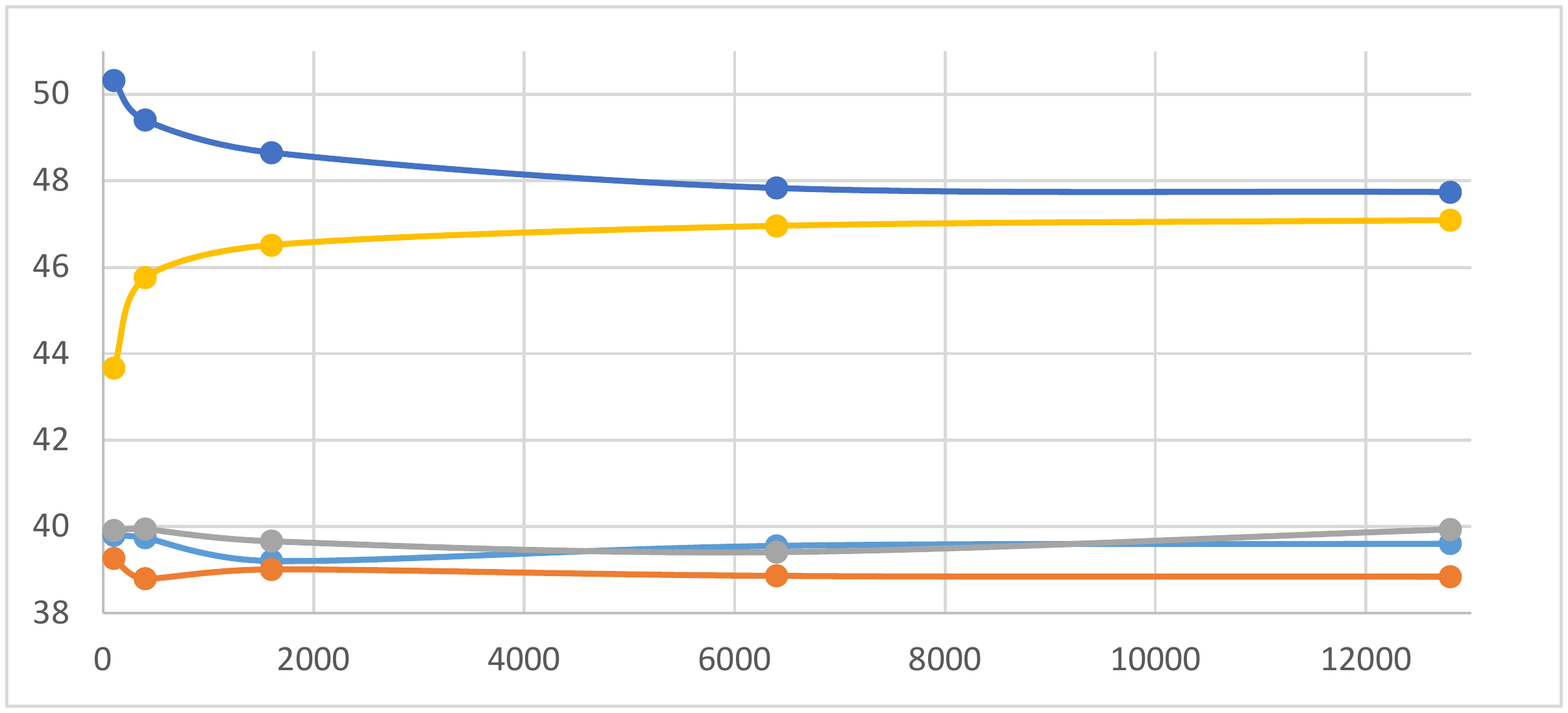}}
    
  \end{subfigure}
\vspace*{-0.15in}
\caption{
\textbf{Comparison the different algorithms:} The
    dimension of the problem is presented in the X axis and the accuracy
    on the Y axis.
  \vspace*{-0.2in}
}
\end{figure}

We conducted several experiments with synthetic data to compare 
different methods when clustering in high dimensional spaces. The
synthetic data was generated from three Gaussians with centers at points $\left(0,0\right),\,\left(3,0\right)$
and $\left(-3,+3\right)$. On top of that, we added
additional dimensions with normally distributed noise. In the first
experiment we used unit variance for all dimensions. In the
second experiment we skewed the distribution so that the variance in the
other features is 5. 

Two sets of mixing coefficients for the three Gaussians were chosen
at random 100 times by selecting three uniform values from $\left[0,\,1\right]$
and normalizing them to sum to $1$. We generated two samples
with 80 examples each from the two mixing coefficients. The DSC
and MSP algorithm received these two samples as inputs while the reference
algorithms, which are not designed to use multiple samples, received
the combined set of 160 points as input.

We ran 100 trials.  In each trial, each of the algorithms finds 3 Gaussians. We then 
measure the percentage of the points associated with
the true originating Gaussian after making the best assignment of the inferred centers to the true Gaussians.

We compared several algorithms. K-means was used on the data
as a baseline. We compared three low dimensional projection algorithms.
Following \cite{Dasgupta99} we used random projections as the first
of these. Second, following \cite{VempalaWang02} we used PCA to project on
the maximal variance subspace. MSP was used as the third projection
algorithm. In all projection algorithm we first projected on
a one dimensional space and then applied K-means to find the clusters.
Finally, we used the DSC algorithm. The DSC algorithm uses the
classregtree function in MATLAB as its learning oracle. Whenever 
K-means was applied, the MATLAB implementation of this procedure was
used with 10 random initial starts.

Figure~\ref{fig:spherical} shows the results of the first experiment
with unit  variance in the noise dimensions. In this setting,
the Maximal Variance method is expected to work well since the first
two dimensions have larger expected variance. Indeed we see that this
is the case. However, when the number of dimensions is large, MSP and DSC
 outperform the other methods; this corresponds to the difficult regime of low signal to noise ratio. In $12800$ dimensions, MSP outperforms
Random Projections $90\%$ of the time, Maximal Variance $80\%$
of the time, and  K-means $79\%$ of the time. DSC outperforms
Random Projections, Maximal Variance and K-means $84\%$, $69\%$, and $66\%$ of the time
respectively. Thus the p-value in all these experiments is $<0.01$.

Figure~\ref{fig:skewed} shows the results of the experiment in which
 the variance in the noise dimensions is higher which creates a more challanging problem.
 In this case,
we see that all the reference methods suffer significantly, but the
MSP and the DSC methods obtain similar results as in the previous
setting. Both the MSP and the DSC algorithms win over Random Projections,
Maximal Variance and K-means more than $78\%$ of the time when the dimension
is $400$ and up. The p-value of these experiments is
$<1.6\times10^{-7}$.
\section{Conclusions}
The mixture problem examined here is closely related to the problem of clustering. 
Most clustering data can be viewed as points generated from multiple underlying 
distributions or generating functions, and clustering can be seen as the process of 
recovering the structure of or assignments to these distributions.
We presented two algorithms for the mixture problem that can be viewed as clustering 
algorithms. The MSP algorithm uses multiple
samples to find a low dimensional space to project the data to.
The DSC algorithm builds a clustering tree assuming that the clusters
are disjoint. We proved that these algorithms work under milder assumptions than 
currently known methods.  The key message in this work is that when multiple samples 
are available, often it is best not to pool the data into one large sample, but that 
the structure in the different samples can be leveraged to improve clustering power.

\bibliographystyle{amsplain}
\bibliography{multiSample}

\shortVersion{
\appendix
\section{Supplementary Material}

Here we provide detailed analysis of the results presented in the
paper that could not fit due to space limitations.

\subsection{Proof of Lemma~\ref{lemma: DSC full knowledge}}

\begin{proof}
\begin{enumerate}
\item Let $A$ be a measurable set let $A_{i}=A\cap C_{i}$. 
\begin{eqnarray*}
D_{1}\left(A\right)-D_{2}\left(A\right) & = & \sum_{i}D_{1}\left(A_{i}\right)-D_{2}\left(A_{i}\right)\\
 & = & \sum_{i}\phi_{i}^{1}\theta_{i}\left(A_{i}\right)-\phi_{i}^{2}\theta_{i}\left(A_{i}\right)\\
 & = & \sum_{i}\theta_{i}\left(A_{i}\right)\left(\phi_{i}^{1}-\phi_{i}^{2}\right)\\
 & \leq & \sum_{i}\max\left(\phi_{i}^{1}-\phi_{i}^{2},0\right)\,\,\,.
\end{eqnarray*}

\item Let $A^{*}=\cup_{i:\phi_{i}^{1}>\phi_{i}^{2}}C_{i}$ then
\begin{eqnarray*}
D_{1}\left(A^{*}\right)-D_{2}\left(A^{*}\right) & = & \sum_{i:\phi_{i}^{1}>\phi_{i}^{2}}\theta_{i}\left(C_{i}\right)\left(\phi_{i}^{1}-\phi_{i}^{2}\right)\\
 & = & \sum_{i}\max\left(\phi_{i}^{1}-\phi_{i}^{2},0\right)\,\,\,.
\end{eqnarray*}

\item Assume that $\forall i,\,\,\phi_{i}^{1}\neq\phi_{i}^{2}$ and $A$
is such that $D_{1}\left(A\right)-D_{2}\left(A\right)=L_{1}\left(D_{1},D_{2}\right)$.
As before let $A_{i}=A\cap C_{i}$ then
\[
D_{1}\left(A\right)-D_{2}\left(A\right)=\sum_{i}\theta_{i}\left(A_{i}\right)\left(\phi_{i}^{1}-\phi_{i}^{2}\right)\,\,\,.
\]
If exists $i$ such that $\theta_{i}\left(A\Delta A^{*}\right)\neq0$
then there could be two cases. If $i^{*}$ is such that $\phi_{i^{*}}^{1}>\phi_{i^{*}}^{2}$
then $\theta_{i^{*}}\left(A^{*}\right)=1$ hence $\theta_{i^{*}}\left(A\right)<1$.
Therefore,
\begin{eqnarray*}
D_{1}\left(A\right)-D_{2}\left(A\right) & \leq & \sum_{i}\theta_{i}\left(A_{i}\right)\max\left(\phi_{i}^{1}-\phi_{i}^{2},0\right)\\
 & < & \sum_{i}\max\left(\phi_{i}^{1}-\phi_{i}^{2},0\right)
\end{eqnarray*}
which contradicts the assumptions. In the same way, if $i^{*}$ is
such that $\phi_{i^{*}}^{1}<\phi_{i^{*}}^{2}$ then $\theta_{i^{*}}\left(A^{*}\right)=0$
hence $\theta_{i^{*}}\left(A\right)>0$. Therefore,
\begin{eqnarray*}
D_{1}\left(A\right)-D_{2}\left(A\right) & \leq & \sum_{i}\theta_{i}\left(A_{i}\right)\max\left(\phi_{i}^{1}-\phi_{i}^{2},0\right)+\theta_{i^{*}}\left(A_{i^{*}}\right)\left(\phi_{i}^{1}-\phi_{i}^{2}\right)\\
 & < & \sum_{i}\theta_{i}\left(A_{i}\right)\max\left(\phi_{i}^{1}-\phi_{i}^{2},0\right)
\end{eqnarray*}

\end{enumerate}
\end{proof}

\subsection{Proof of MSP Analysis theorem}

\begin{proof} \textbf{of Theorem~\ref{thm: MSP-Analysis}}
\begin{enumerate}
\item The computational complexity is straight forward. The MSP algorithm
first computes the expected value for each of the samples. For every
sample this takes $nN_{j}$. Once the expected values were computed,
computing each of the $\bar{v}_{j}$ vector is $2n$ operations.
\item Recall that $D_{j}=\sum_{i}\phi_{i}^{j}\theta_{i}$. We can rewrite
it as
\[
D_{j}=\sum_{i=1}^{K}\phi_{i}^{j}\mu_{i}+\sum_{i=1}^{K}\phi_{i}^{j}\left(\theta_{i}-\mu_{i}\right)=E_{j}+\sum_{i=1}^{K}\phi_{i}^{j}\left(\theta_{i}-\mu_{i}\right)\,\,\,.
\]

Note that for every $i$, $\left(\theta_{i}-\mu_{i}\right)$ has a
zero mean and variance bounded by $\sigmamax^{2}$. Since $\phi_{i}^{j}\geq0$
and $\sum_{i}\phi_{i}^{j}=1$ then the measure $\sum_{i=1}^{K}\phi_{i}^{j}\left(\theta_{i}-\mu_{i}\right)$
has zero mean and variance bounded by $\sigmamax^{2}$. Hence, $D_{j}$
is a measure with mean $E_{j}$ and variance bounded by $\sigmamax^{2}$.
Since $\bar{E}_{j}$ is obtained by averaging $N_{j}$ instances,
we get, from Chebyshev's inequality combined with the union bound
that
\[
\Pr\left[\norm{E_{j}-\bar{E}_{j}}>\epsilon\right]\leq\frac{n\sigmamax^{2}}{N_{j}\epsilon^{2}}\,\,\,.
\]
Since there are $m$ estimators, $\bar{E}_{1},\ldots,\bar{E}_{m}$,
using the union bound we obtain
\[
\Pr\left[\sup_{j}\norm{E_{j}-\bar{E}_{j}}>\epsilon\right]\leq\sum_{j}\frac{n\sigmamax^{2}}{N_{j}\epsilon^{2}}\,\,\,.
\]

\item Recall that $\bar{\mu}_{i}$ is the projection of $\mu_{i}$ on the
space $\bar{V}=\mbox{span}\left(\bar{v}_{1},\ldots,\bar{v}_{m-1}\right)$.
Therefore, $\bar{\mu}_{i}=\arg\min_{v\in\bar{V}}\left(\norm{\mu_{i}-v}\right).$
Since $\mu_{i}\in\mbox{span}\left\{ v_{1},\ldots,v_{m-1}\right\} $
then $\mu_{i}=\sum\alpha_{j}^{i}v_{j}$. 
\begin{eqnarray*}
\norm{\mu_{i}-\bar{\mu}_{i}} & \leq & \norm{\sum\alpha_{j}v_{j}-\sum\alpha_{j}\bar{v}_{j}}\\
 & \leq & \sum\left|\alpha\right|_{j}\norm{v_{j}-\bar{v}_{j}}\,\,\,.
\end{eqnarray*}
Hence, if $A=\max_{i}\sum\left|\alpha_{j}^{i}\right|$ then with probability
of at least $1-\frac{n\sigmamax^{2}}{\epsilon^{2}}\sum_{j}\frac{1}{n_{j}}$
\[
\max_{i}\norm{\mu_{i}-\bar{\mu}_{i}}\leq\epsilon A\,\,\,.
\]
Furthermore,
\[
\max_{i,i^{\prime}}\left|\norm{\mu_{i}-\mu_{i^{\prime}}}-\norm{\bar{\mu}_{i}-\bar{\mu}_{i^{\prime}}}\right|\leq2\epsilon A\,\,\,.
\]

\end{enumerate}
\end{proof}

It is possible to improve upon the bounds presented here. We can get
sharper bounds on the probability of success in several ways:
\begin{enumerate}
\item If we assume that the sample space is bounded we can use Bernstein's
inequality instead of Chebyshev's inequality
\item If we assume that the covariance matrix is diagonal we can replace
the union bounded with better concentration of measure bounds
\item If we assume that the distributions are Gaussians we can use tail
bounds on these distribution instead of Chebyshev's inequality
\end{enumerate}
To simplify the presentation, we do not derive the bounds for these
specific conditions.

\subsection{Proof of Theorem~\ref{Thm: DSC}}

The analysis we use several lemmas. To simplify the notation we define
the following assumptions:

\begin{assumption} The gap between $\Phi^{1}$ and $\Phi^{2}$ is
$g>0$. \label{assumption: gap}

\end{assumption}

\begin{assumption} Both $\Phi^{1}$ and $\Phi^{2}$ are bounded away
from zero by $b>0$.\label{assumption: bounded from zero}

\end{assumption}

\begin{definition} For the set $A$, we say that the $i$'th cluster
is $\gamma$-\emph{big} if $\theta_{i}\left(A\right)\geq1-\gamma$
and we say that it is $\gamma$-\emph{small} if $\theta_{i}\left(A\right)\leq\gamma$.\end{definition}

\begin{assumption} For the set $A$ all clusters are either $\gamma$-big
or $\gamma$-small and there exists at least one $\gamma$-big cluster.\label{assumption: gamma}

\end{assumption}

\begin{assumption} The classifier $h$ and the estimator $e$ are
such that\label{assumption h and e}
\[
\left|\Pr_{x,y\sim\nu}\left[h\left(x\right)\neq y\right]-e\right|\leq\epsilon
\]
and 
\[
\forall I,\,\,\,\,\Pr_{x,y\sim\nu}\left[h\left(x\right)\neq y\right]<\Pr_{x,y\sim\nu}\left[c_{I}\left(x\right)\neq y\right]+\epsilon
\]
where $\nu$ is a measure on $X\times\left\{ \pm1\right\} $ such that
the measure of $B\subseteq X\times\left\{ \pm1\right\} $ is
\[
\nu\left(B\right)=\frac{1}{2}\left(\frac{D_{1}\left(\left\{ x\in A\,\,:\,\,\left(x,1\right)\in B\right\} \right)}{D_{1}\left(A\right)}+\frac{D_{2}\left(\left\{ x\in A\,\,:\,\,\left(x,-1\right)\in B\right\} \right)}{D_{2}\left(A\right)}\right)\,\,\,.
\]

\end{assumption}

Using these assumptions we turn to prove the lemmas.

\begin{lemma} \label{lemma cI} Under assumptions \ref{assumption: bounded from zero}
and \ref{assumption: gamma}, if $I=\left\{ i\,:\,\frac{\phi_{i}^{1}}{\sum_{i^{\prime}}\phi_{i^{\prime}}^{1}\theta_{i^{\prime}}\left(A\right)}>\frac{\phi_{i}^{2}}{\sum_{i^{\prime}}\phi_{i^{\prime}}^{2}\theta_{i^{\prime}}\left(A\right)}\right\} $,
there are more than a single $\gamma$-big cluster and $\gamma<\min\left(\nicefrac{b}{2},\,\nicefrac{gb}{k+3}\right)$
then
\[
\Pr_{x,y\sim\nu}\left[c_{I}\left(x\right)\neq y\right]\leq\frac{1}{2}\left(1-g\left(1-\gamma\right)+\frac{3K\gamma}{b}\right)
\]
where $\nu$ is as defined in assumption~\ref{assumption h and e}.
Moreover, the set $I$ contains at least one $\gamma$-big cluster
but does not contain all the $\gamma$-big clusters.\end{lemma}

\begin{proof} Let $J$ be the set of $\gamma$ big clusters. From
the definition of $\gamma$ we have that
\begin{eqnarray*}
\Pr_{x,y\sim\nu}\left[c_{I}\left(x\right)\neq y\right] & = & \frac{1}{2}\left(\sum_{i\notin I}\frac{\phi_{i}^{1}\theta_{i^{\prime}}\left(A\right)}{\sum_{i^{\prime}}\phi_{i^{\prime}}^{1}\theta_{i^{\prime}}\left(A\right)}+\sum_{i\in I}\frac{\phi_{i}^{2}\theta_{i^{\prime}}\left(A\right)}{\sum_{i^{\prime}}\phi_{i^{\prime}}^{2}\theta_{i^{\prime}}\left(A\right)}\right)\\
 & = & \frac{1}{2}\left(1-\sum_{i\in I}\left(\frac{\phi_{i}^{1}}{\sum_{i^{\prime}}\phi_{i^{\prime}}^{1}\theta_{i^{\prime}}\left(A\right)}-\frac{\phi_{i}^{2}}{\sum_{i^{\prime}}\phi_{i^{\prime}}^{2}\theta_{i^{\prime}}\left(A\right)}\right)\theta_{i^{\prime}}\left(A\right)\right)\\
 & \leq & \frac{1}{2}\left(1-\sum_{i\in I\cap J}\left(\frac{\phi_{i}^{1}}{\sum_{i^{\prime}}\phi_{i^{\prime}}^{1}\theta_{i^{\prime}}\left(A\right)}-\frac{\phi_{i}^{2}}{\sum_{i^{\prime}}\phi_{i^{\prime}}^{2}\theta_{i^{\prime}}\left(A\right)}\right)\theta_{i^{\prime}}\left(A\right)\right)\\
 & \leq & \frac{1}{2}\left(1-\sum_{i\in I\cap J}\left(\frac{\phi_{i}^{1}}{\sum_{i^{\prime}\in J}\phi_{i^{\prime}}^{1}+\gamma}-\frac{\phi_{i}^{2}}{\sum_{i^{\prime}\in J}\phi_{i^{\prime}}^{2}-\gamma}\right)\theta_{i^{\prime}}\left(A\right)\right)\,\,\,.
\end{eqnarray*}

Due to assumption \ref{assumption: bounded from zero} 
\begin{eqnarray}
\frac{\phi_{i}^{1}}{\sum_{i^{\prime}\in J}\phi_{i^{\prime}}^{1}+\gamma} & = & \frac{\phi_{i}^{1}}{\sum_{i^{\prime}\in J}\phi_{i^{\prime}}^{1}}-\frac{\gamma\phi_{i}^{i}}{\left(\sum_{i^{\prime}\in J}\phi_{i^{\prime}}^{1}\right)\left(\sum_{i^{\prime}\in J}\phi_{i^{\prime}}^{1}+\gamma\right)}\nonumber \\
 & \geq & \frac{\phi_{i}^{1}}{\sum_{i^{\prime}\in J}\phi_{i^{\prime}}^{1}}-\frac{\gamma}{\sum_{i^{\prime}\in J}\phi_{i^{\prime}}^{1}+\gamma}\nonumber \\
 & \geq & \frac{\phi_{i}^{1}}{\sum_{i^{\prime}\in J}\phi_{i^{\prime}}^{1}}-\frac{\gamma}{b}\,\,\,.\label{eq: lower}
\end{eqnarray}

Since $\gamma<\nicefrac{b}{2}$ we have
\begin{eqnarray}
\frac{\phi_{i}^{2}}{\sum_{i^{\prime}\in J}\phi_{i^{\prime}}^{2}-\gamma} & = & \frac{\phi_{i}^{2}}{\sum_{i^{\prime}\in J}\phi_{i^{\prime}}^{2}}+\frac{\gamma\phi_{i}^{2}}{\left(\sum_{i^{\prime}\in J}\phi_{i^{\prime}}^{2}\right)\left(\sum_{i^{\prime}\in J}\phi_{i^{\prime}}^{2}-\gamma\right)}\nonumber \\
 & \leq & \frac{\phi_{i}^{2}}{\sum_{i^{\prime}\in J}\phi_{i^{\prime}}^{2}}+\frac{\gamma}{\sum_{i^{\prime}\in J}\phi_{i^{\prime}}^{2}-\gamma}\nonumber \\
 & \leq & \frac{\phi_{i}^{2}}{\sum_{i^{\prime}\in J}\phi_{i^{\prime}}^{2}}+\frac{2\gamma}{b}\,\,\,.\label{eq: upper}
\end{eqnarray}

Therefore,
\begin{eqnarray*}
\Pr_{x,y\sim\nu}\left[c_{I}\left(x\right)\neq y\right] & \leq & \frac{1}{2}\left(1-\sum_{i\in I\cap J}\left(\frac{\phi_{i}^{1}}{\sum_{i^{\prime}\in J}\phi_{i^{\prime}}^{1}+\gamma}-\frac{\phi_{i}^{2}}{\sum_{i^{\prime}\in J}\phi_{i^{\prime}}^{2}-\gamma}\right)\theta_{i}\left(A\right)\right)\\
 & \leq & \frac{1}{2}\left(1-\sum_{i\in I\cap J}\left(\frac{\phi_{i}^{1}}{\sum_{i^{\prime}\in J}\phi_{i^{\prime}}^{1}}-\frac{\phi_{i}^{2}}{\sum_{i^{\prime}\in J}\phi_{i^{\prime}}^{2}}-\frac{3\gamma}{b}\right)\theta_{i}\left(A\right)\right)\\
 & \leq & \frac{1}{2}\left(1-\sum_{i\in I\cap J}\left(\frac{\phi_{i}^{1}}{\sum_{i^{\prime}\in J}\phi_{i^{\prime}}^{1}}-\frac{\phi_{i}^{2}}{\sum_{i^{\prime}\in J}\phi_{i^{\prime}}^{2}}\right)\theta_{i}\left(A\right)+\frac{3K\gamma}{b}\right)\\
 & \leq & \frac{1}{2}\left(1-\sum_{i\in I\cap J}g\theta_{i}\left(A\right)+\nicefrac{gb}{k+3}\frac{3K\gamma}{b}\right)\\
 & \leq & \frac{1}{2}\left(1-g\left(1-\gamma\right)+\frac{3K\gamma}{b}\right)\,\,\,.
\end{eqnarray*}
Note that we have used the fact that $I\cap J$ is not empty. Otherwise,
note that since $\gamma<\nicefrac{1}{2}$ and from (\ref{eq: lower})
and (\ref{eq: upper})
\begin{eqnarray*}
0 & = & \sum_{i\in I}\left(\frac{\phi_{i}^{1}\theta_{i^{\prime}}\left(A\right)}{\sum_{i^{\prime}}\phi_{i^{\prime}}^{1}\theta_{i^{\prime}}\left(A\right)}-\frac{\phi_{i}^{2}\theta_{i^{\prime}}\left(A\right)}{\sum_{i^{\prime}}\phi_{i^{\prime}}^{2}\theta_{i^{\prime}}\left(A\right)}\right)-\sum_{i\notin I}\left(\frac{\phi_{i}^{2}\theta_{i^{\prime}}\left(A\right)}{\sum_{i^{\prime}}\phi_{i^{\prime}}^{2}\theta_{i^{\prime}}\left(A\right)}-\frac{\phi_{i}^{1}\theta_{i^{\prime}}\left(A\right)}{\sum_{i^{\prime}}\phi_{i^{\prime}}^{1}\theta_{i^{\prime}}\left(A\right)}\right)\\
 & \leq & \gamma\sum_{i\in I}\left(\frac{\phi_{i}^{1}}{\sum_{i^{\prime}}\phi_{i^{\prime}}^{1}\theta_{i^{\prime}}\left(A\right)}\right)-\left(1-\gamma\right)\sum_{i\in J}\left(\frac{\phi_{i}^{2}}{\sum_{i^{\prime}}\phi_{i^{\prime}}^{2}\theta_{i^{\prime}}\left(A\right)}-\frac{\phi_{i}^{1}}{\sum_{i^{\prime}}\phi_{i^{\prime}}^{1}\theta_{i^{\prime}}\left(A\right)}\right)\\
 & \leq & \gamma K\frac{1}{2b\left(1-\gamma\right)}-\left(1-\gamma\right)\sum_{i\in J}\left(g-\frac{3\gamma}{b}\right)\\
 & \leq & \frac{\gamma K}{b}-2\left(1-\gamma\right)\left(g-\frac{3\gamma}{b}\right)\\
 & \leq & \frac{\gamma K}{b}-\left(g-\frac{3\gamma}{b}\right)\\
 & = & \gamma\frac{K+3}{b}-g\,\,\,.
\end{eqnarray*}
However, since $\gamma<\nicefrac{gb}{k+3}$ we obtain $\gamma\frac{K+3}{b}-g<0$
which is a contradiction. Therefore, $I\cap J$ is not empty. In the
same way we can see that $\bar{I}\cap J$ is not empty as well.\end{proof}

\begin{lemma} \label{lemma: gamma}Under assumptions \ref{assumption: gap},
\ref{assumption: bounded from zero}, \ref{assumption: gamma} and
\ref{assumption h and e} if $\gamma<\min\left(\nicefrac{b}{2},\nicefrac{gb}{3}\right)$
then 
\[
\forall i,\,\,\,\theta_{i}\left(A\cap h\right),\,\theta_{i}\left(A\cap\bar{h}\right)\notin\left[\gamma+\frac{2\epsilon}{g-\frac{3\gamma}{b}},\,1-\gamma-\frac{2\epsilon}{g-\frac{3\gamma}{b}}\right]\,\,\,.
\]
Moreover, if $I=\left\{ i\,:\,\frac{\phi_{i}^{1}}{\sum_{i}\phi_{i}^{1}\theta_{i}\left(A\right)}>\frac{\phi_{i}^{2}}{\sum_{i}\phi_{i}^{2}\theta_{i}\left(A\right)}\right\} $
then $\theta_{i}\left(A\cap h\right)\geq1-\gamma-\frac{2\epsilon}{g-\frac{3\gamma}{b}}$
iff $i$ is a $\gamma$-big cluster in $A$ and $i\in I$.\end{lemma}

\begin{proof} Recall that for every set $B$, $D_{j}\left(B\right)=\sum_{i}\phi_{i}^{j}\theta_{i}\left(B\right)$.
By using $h$ both as a function and as a subset of $X$ we have
\begin{eqnarray*}
\Pr_{x,y\sim\nu}\left[h\left(x\right)\neq y\right] & = & \frac{1}{2}\left(\frac{D^{1}\left(A\setminus h\right)}{D^{1}\left(A\right)}+\frac{D^{2}\left(h\right)}{D^{2}\left(A\right)}\right)\\
 & = & \frac{1}{2}\left(1-\left(\frac{D^{1}\left(h\cap A\right)}{D^{1}\left(A\right)}-\frac{D^{2}\left(h\cap A\right)}{D^{2}\left(A\right)}\right)\right)\\
 & = & \frac{1}{2}\left(1-\left(\frac{\sum_{i}\phi_{i}^{1}\theta_{i}\left(h\cap A\right)}{\sum_{i}\phi_{i}^{1}\theta_{i}\left(A\right)}-\frac{\sum_{i}\phi_{i}^{2}\theta_{i}\left(h\cap A\right)}{\sum_{i}\phi_{i}^{2}\theta_{i}\left(A\right)}\right)\right)\\
 & = & \frac{1}{2}\left(1-\sum_{i}\left(\frac{\phi_{i}^{1}}{\sum_{i^{\prime}}\phi_{i^{\prime}}^{1}\theta_{i^{\prime}}\left(A\right)}-\frac{\phi_{i}^{2}}{\sum_{i^{\prime}}\phi_{i^{\prime}}^{2}\theta_{i^{\prime}}\left(A\right)}\right)\theta_{i}\left(h\cap A\right)\right)
\end{eqnarray*}
Let $I=\left\{ i\,:\,\frac{\phi_{i}^{1}}{\sum_{i}\phi_{i}^{1}\theta_{i}\left(A\right)}>\frac{\phi_{i}^{2}}{\sum_{i}\phi_{i}^{2}\theta_{i}\left(A\right)}\right\} $.
It follows that $c_{I}=\arg\min_{h}\Pr_{x,y\sim\nu}\left[h\left(x\right)\neq y\right]$
and hence
\[
\Pr_{x,y\sim\nu}\left[c_{I}\left(x\right)\neq y\right]\leq\Pr_{x,y\sim\nu}\left[h\left(x\right)\neq y\right]<\Pr_{x,y\sim\nu}\left[c_{I}\left(x\right)\neq y\right]+\epsilon\,\,\,.
\]
Therefore,
\begin{eqnarray*}
\epsilon & > & \Pr_{x,y\sim\nu}\left[h\left(x\right)\neq y\right]-\Pr_{x,y\sim\nu}\left[c_{I}\left(x\right)\neq y\right]\\
 & = & \frac{1}{2}\left(\sum_{i}\left(\ind_{i\in I}\theta_{i}\left(A\right)-\theta_{i}\left(h\cap A\right)\right)\left(\frac{\phi_{i}^{1}}{\sum_{i}\phi_{i}^{1}\theta_{i}\left(A\right)}-\frac{\phi_{i}^{2}}{\sum_{i}\phi_{i}^{2}\theta_{i}\left(A\right)}\right)\right)\\
 & \geq & \frac{1}{2}\max_{i}\left(\ind_{i\in I}\theta_{i}\left(A\right)-\theta_{i}\left(h\cap A\right)\right)\left(\frac{\phi_{i}^{1}}{\sum_{i}\phi_{i}^{1}\theta_{i}\left(A\right)}-\frac{\phi_{i}^{2}}{\sum_{i}\phi_{i}^{2}\theta_{i}\left(A\right)}\right)
\end{eqnarray*}
Let $J=\left\{ j\,:\,\theta_{j}\left(A\right)>1-\gamma\right\} $
then
\begin{eqnarray*}
\frac{\phi_{i}^{1}}{\sum_{i}\phi_{i}^{1}\theta_{i}\left(A\right)}-\frac{\phi_{i}^{2}}{\sum_{i}\phi_{i}^{2}\theta_{i}\left(A\right)} & \geq & \frac{\phi_{i}^{1}}{\sum_{j\in J}\phi_{j}^{1}+\sum_{j\in J}\phi_{j}^{1}\gamma}-\frac{\phi_{i}^{2}}{\sum_{j\in J}\phi_{j}^{2}\left(1-\gamma\right)}\\
 & \geq & \frac{\phi_{i}^{1}}{\sum_{j\in J}\phi_{j}^{1}+\gamma}-\frac{\phi_{i}^{2}}{\sum_{j\in J}\phi_{j}^{2}-\gamma}
\end{eqnarray*}
and hence
\[
\epsilon\geq\frac{1}{2}\max_{i}\left(\ind_{i\in I}\theta_{i}\left(A\right)-\theta_{i}\left(h\cap A\right)\right)\left(\frac{\phi_{i}^{1}}{\sum_{j\in J}\phi_{j}^{1}+\gamma}-\frac{\phi_{i}^{2}}{\sum_{j\in J}\phi_{j}^{2}-\gamma}\right)
\]
or put otherwise
\begin{equation}
\max_{i}\left(\ind_{i\in I}\theta_{i}\left(A\right)-\theta_{i}\left(h\cap A\right)\right)\leq\frac{2\epsilon}{\min_{i}\left(\frac{\phi_{i}^{1}}{\sum_{j\in J}\phi_{j}^{1}+\gamma}-\frac{\phi_{i}^{2}}{\sum_{j\in J}\phi_{j}^{2}-\gamma}\right)}\,\,\,.\label{eq:gap}
\end{equation}

Note that since $\gamma<\nicefrac{b}{2}$ from (\ref{eq: lower})
and (\ref{eq: upper}) it follows that
\begin{eqnarray*}
\min_{i}\left(\frac{\phi_{i}^{1}}{\sum_{j\in J}\phi_{j}^{1}+\gamma}-\frac{\phi_{i}^{2}}{\sum_{j\in J}\phi_{j}^{2}-\gamma}\right) & \geq & \min_{i}\left(\frac{\phi_{i}^{1}}{\sum_{j\in J}\phi_{j}^{1}}-\frac{\phi_{i}^{2}}{\sum_{j\in J}\phi_{j}^{2}}\right)-\frac{3\gamma}{b}\\
 & \geq & g-\frac{3\gamma}{b}\,\,\,.
\end{eqnarray*}
Note that $g-\nicefrac{3\gamma}{b}>0$ since $\gamma<\nicefrac{gb}{3}$.
If $i$ is such that $\theta_{i}\left(A\right)<\gamma$ then clearly
$\theta_{i}\left(h\cap A\right)<\gamma$. However, if $\theta_{i}\left(A\right)>1-\gamma$
and $i\in I$ then from (\ref{eq:gap}) we have that 
\begin{eqnarray*}
\theta_{i}\left(h\cap A\right) & \geq & \theta_{i}\left(A\right)-\frac{2\epsilon}{g-\frac{3\gamma}{b}}\,\,\,.\\
 & \geq & 1-\gamma-\frac{2\epsilon}{g-\frac{3\gamma}{b}}
\end{eqnarray*}

If however, $\theta_{i}\left(A\right)>1-\gamma$ but $i\notin I$
then we can repeat the same argument for $\bar{h}$ to get
\[
\theta_{i}\left(\bar{h}\cap A\right)\geq1-\gamma-\frac{2\epsilon}{g-\frac{3\gamma}{b}}
\]
and thus
\[
\theta_{i}\left(h\cap A\right)\leq\gamma+\frac{2\epsilon}{g-\frac{3\gamma}{b}}\,\,\,.
\]

\end{proof}

\begin{lemma} \label{lemma: single big cluster}Under assumptions
\ref{assumption: bounded from zero},\ref{assumption: gamma} and
\ref{assumption h and e}, if there is only a single $\gamma$-big
cluster then
\[
e\geq\frac{1}{2}-\gamma\left(\frac{1}{\gamma+b\left(1-\gamma\right)}+\frac{1}{1-\gamma}+\frac{1}{b+\gamma}\right)-\epsilon
\]

\end{lemma}

\begin{proof} Let $i^{*}$ be such that $\theta_{i^{*}}$ is the
single $\gamma$-big cluster. For $j=1,2$ 
\begin{eqnarray*}
\frac{D_{j}\left(C_{i^{*}}\cap A\right)}{D_{j}\left(A\right)} & = & \frac{\phi_{i^{*}}^{j}\theta_{i^{*}}\left(A\right)}{\sum_{i}\phi_{i}^{j}\theta_{i}\left(A\right)}\\
 & = & \frac{1}{1+\frac{\sum_{i\neq i^{*}}\phi_{i}^{j}\theta_{i}\left(A\right)}{\phi_{i^{*}}^{j}\theta_{i^{*}}\left(A\right)}}\\
 & \geq & \frac{1}{1+\frac{\gamma}{b\left(1-\gamma\right)}}\,\,\,.
\end{eqnarray*}

For any $h$, 
\begin{eqnarray*}
\left|\frac{D_{j}\left(A\cap h\right)}{D_{j}\left(A\right)}-\theta_{i^{*}}\left(h\right)\right| & \leq & \left|\frac{D_{j}\left(\left(A\setminus C_{i^{*}}\right)\cap h\right)}{D_{j}\left(A\right)}\right|+\left|\frac{D_{j}\left(A\cap C_{i^{*}}\cap h\right)}{D_{j}\left(A\right)}-\theta_{i^{*}}\left(h\right)\right|\\
 & \leq & \frac{\gamma}{\gamma+b\left(1-\gamma\right)}+\left|\frac{\phi_{i^{*}}^{j}\theta_{i^{*}}\left(A\cap h\right)}{D_{j}\left(A\right)}-\theta_{i^{*}}\left(h\right)\right|\\
 & \leq & \frac{\gamma}{\gamma+b\left(1-\gamma\right)}+\left|\frac{\phi_{i^{*}}^{j}\theta_{i^{*}}\left(\bar{A}\right)}{D_{j}\left(A\right)}\right|+\theta_{i^{*}}\left(h\right)\left|\frac{\phi_{i^{*}}^{j}}{D_{j}\left(A\right)}-1\right|\\
 & \leq & \frac{\gamma}{\gamma+b\left(1-\gamma\right)}+\frac{\gamma}{1-\gamma}+\theta_{i^{*}}\left(h\right)\left(1-\frac{\phi_{i^{*}}^{j}}{\phi_{i^{*}}^{j}+\gamma}\right)\\
 & \leq & \frac{\gamma}{\gamma+b\left(1-\gamma\right)}+\frac{\gamma}{1-\gamma}+\frac{\gamma}{b+\gamma}\,\,\,.
\end{eqnarray*}
Therefore,
\begin{eqnarray*}
\left|\frac{D_{1}\left(A\cap h\right)}{D_{1}\left(A\right)}-\frac{D_{2}\left(A\cap h\right)}{D_{2}\left(A\right)}\right| & \leq & \left|\frac{D_{1}\left(A\cap h\right)}{D_{1}\left(A\right)}-\theta_{i^{*}}\left(h\right)\right|+\left|\frac{D_{2}\left(A\cap h\right)}{D_{2}\left(A\right)}-\theta_{i^{*}}\left(h\right)\right|\\
 & \leq & 2\gamma\left(\frac{1}{\gamma+b\left(1-\gamma\right)}+\frac{1}{1-\gamma}+\frac{1}{b+\gamma}\right)\,\,\,.
\end{eqnarray*}

Hence
\begin{eqnarray*}
e & \geq & \Pr_{x,y\sim\nu}\left[h\left(x\right)\neq y\right]-\epsilon\\
 & = & \frac{1}{2}-\frac{1}{2}\left|\frac{D_{1}\left(A\cap h\right)}{D_{1}\left(A\right)}-\frac{D_{2}\nicefrac{gb}{3}\left(A\cap h\right)}{D_{2}\left(A\right)}\right|-\epsilon\\
 & \geq & \frac{1}{2}-\gamma\left(\frac{1}{\gamma+b\left(1-\gamma\right)}+\frac{1}{1-\gamma}+\frac{1}{b+\gamma}\right)-\epsilon\,\,\,.
\end{eqnarray*}

\end{proof}

\begin{lemma} \label{lemma: multiple big clusters}Under assumptions
\ref{assumption: bounded from zero},\ref{assumption: gamma} and
\ref{assumption h and e}, if there are $t>1$ clusters which are
$\gamma$-big for $\gamma<\min\left(\nicefrac{b}{2},\,\nicefrac{gb}{k+3}\right)$
then
\[
e\leq\frac{1}{2}\left(1-g\left(1-\gamma\right)+\frac{3K\gamma}{b}\right)+2\epsilon
\]
and the split induced by $h$ will have at least one $\gamma+\frac{2\epsilon}{g-\frac{3\gamma}{b}}$
big cluster in each side of the split.

\end{lemma}

\begin{proof} Let $I=\left\{ i\,:\,\frac{\phi_{i}^{1}}{\sum_{i}\phi_{i}^{1}\theta_{i}\left(A\right)}>\frac{\phi_{i}^{2}}{\sum_{i}\phi_{i}^{2}\theta_{i}\left(A\right)}\right\} $
then from Lemma~\ref{lemma cI} 
\[
\Pr_{x,y\sim\nu}\left[c_{I}\left(x\right)\neq y\right]\leq\frac{1}{2}\left(1-g\left(1-\gamma\right)+\frac{3K\gamma}{b}\right)
\]
and $I$ contains at least one $\gamma$-big cluster but does not
contain all the $\gamma$ big clusters. From Assumption~\ref{assumption h and e}
it follows that 
\[
e\leq\frac{1}{2}\left(1-g\left(1-\gamma\right)+\frac{3K\gamma}{b}\right)+2\epsilon\,\,\,.
\]
Moreover, from Lemma~\ref{lemma: gamma} a cluster is $\gamma+\frac{2\epsilon}{g-\frac{3\gamma}{b}}$-big
if and only if it is $\gamma$-big in $A$ and it is in $I$. Since
$I$ contains a $\gamma$-big cluster than $A\cap h$ contains a $\gamma-\frac{2\epsilon}{g-\frac{3\gamma}{b}}$-cluster.
However, since $I$ does not contain all the $\gamma$-big clusters,
there is a $\gamma+\frac{2\epsilon}{g-\frac{3\gamma}{b}}$-cluster
in $A\cap\bar{h}$ as well.

\end{proof}

We are now ready to prove the main theorem.

\begin{proof} \textbf{of Theorem~\ref{Thm: DSC}}

Let $\epsilon<\min\left(\frac{g^{2}b}{160K}\,,\,\frac{bg}{8K}\,,\,\frac{g^{2}b}{4K\left(K+3\right)},\,\frac{\epsilon^{*}g}{4K}\right)$
and let $\hat{\gamma}_{l}=\frac{4\epsilon l}{g}$ then for every $l\leq K$
we have 
\[
\frac{3\hat{\gamma}_{l}}{b}<\frac{g}{2}\,\,\,,\,\,\,\hat{\gamma}_{l}<\min\left(\frac{b}{2}\,,\,\frac{gb}{K+3}\right)\,\,\,.
\]

Let $\delta=\frac{\delta^{*}}{4K}$ and $N_{1}$ be the size of the
samples needed for $\alg$ to return an $\epsilon,\delta$ good hypothesis.
Let 
\[
N=\max\left(\frac{4N_{1}}{b}\,,\,\frac{2}{b^{2}}\log\frac{1}{\delta}\right)\,\,\,.
\]

Note the following, if $A$ is a set that contains at least one $\gamma$-big
cluster, for $\gamma\leq\hat{\gamma}_{K}$ then with probability $1-\delta$,
a sample of $N$ points from $\gamma$ contains atleast $N_{1}$ points
from $A$. To see that, note that 
\[
\gamma\leq\frac{4\epsilon K}{g}<\frac{4K}{g}\cdot\frac{gb}{8K}\leq\frac{gb}{2}\leq\frac{1}{2}\,\,\,.
\]
Since each cluster is bounded away from zero by $b>0$, the expected
number points in $A$ out of a sample of size $N$ is at least $\left(1-\gamma\right)bN\geq\nicefrac{bN}{2}$.
From Hoeffding's inequality, we have that for $N\geq\frac{2}{b^{2}}\log\frac{1}{\delta}$
with a probability of $1-\delta$ there will be at least $\nicefrac{bN}{4}$
points in $A$. Since $N\geq\nicefrac{4N_{1}}{b}$, we conclude that
with a probability of at least $1-\delta$, there will be atleast
$N_{1}$ points from $A$ in a random set of $N$ points. Therefore,
with probability $1-\nicefrac{\delta}{2}$, in the first $2K$ calls
for the learning algorithm $\alg$, for which there was at least one
$\gamma$-big cluster in the leaf, there were at least $N_{1}$ points
to train the algorithm from, provided that $\gamma\leq\hat{\gamma}_{K}$.
Hence, with probability $1-\delta$, Assumption~\ref{assumption h and e}
holds for the first $2K$ call to the learning algorithm, provided
that $\gamma\leq\hat{\gamma}_{K}$.

Next we will show that as long as Assumption~\ref{assumption h and e}
holds, the DSC algorithm will make at most $2K$ calls to $\alg$
and all leafs will have at least one $\gamma$-big cluster for $\gamma\leq\hat{\gamma}_{K}$.
We prove that by induction on the depth of the generated tree. Initially,
we have a single leaf with all the points and hence all the clusters
are $0$-big hence the assumption clearly works. From Lemma~\ref{lemma: gamma}
it follows that if all the clusters were $\gamma$-big or $\gamma$-small,
and indeed Assumption~\ref{assumption h and e} holds, then the new
leafs that are generated by splitting on $h$ will have only $\gamma^{\prime}$-big
or $\gamma^{\prime}$-small clusters for $\gamma^{\prime}\leq\gamma+\frac{2\epsilon}{g-\frac{3\gamma}{b}}$.
Note that if $\gamma<\nicefrac{gb}{6}$ then $g-\nicefrac{3\gamma}{b}>\nicefrac{g}{2}$
hence $\gamma^{\prime}\leq\gamma+\nicefrac{4\epsilon}{g}$.

From Lemma~\ref{lemma: single big cluster} and Lemma~\ref{lemma: multiple big clusters}
it follows that every time the DSC algorithm calls the $\alg$ algorithm,
assuming that Assumption~\ref{assumption h and e} holds for this
call, one of two things happen, either the algorithm finds a non-trivial
split of the clusters, which happens whenever there is more than a
single big cluster, in which case $e\leq\frac{1}{2}\left(1-g\left(1-\gamma\right)+\frac{3K\gamma}{b}\right)+2\epsilon$
or otherwise, if there is only a single big cluster, $e\geq\frac{1}{2}-\gamma\left(\frac{1}{\gamma+b\left(1-\gamma\right)}+\frac{1}{1-\gamma}+\frac{1}{b+\gamma}\right)-\epsilon$.
Note that if $\gamma\leq\min\left(\nicefrac{1}{2}\,,\,\nicefrac{4K\epsilon}{g}\right)$
then
\begin{eqnarray*}
\frac{1}{2}-\gamma\left(\frac{1}{\gamma+b\left(1-\gamma\right)}+\frac{1}{1-\gamma}+\frac{1}{b+\gamma}\right)-\epsilon & \geq & \frac{1}{2}-\frac{4K\epsilon}{g}\left(\frac{2}{b}+2+\frac{1}{b}\right)-\epsilon\\
 & \geq & \frac{1}{2}-\frac{K}{gb}\left(12+8b+\frac{gb}{K}\right)\epsilon\\
 & > & \frac{1}{2}-\frac{20K}{gb}\epsilon\,\,\,.
\end{eqnarray*}
Since, $\epsilon<\nicefrac{g^{2}b}{160K}$, if there was only a single
cluster and Assumption~\ref{assumption h and e} holds then $e>\frac{1}{2}-\frac{g}{8}$.
However, if there were multiple big clusters then 
\[
e\leq\frac{1}{2}\left(1-g\left(1-\gamma\right)+\frac{3K\gamma}{b}\right)+2\epsilon\leq\frac{1}{2}-\frac{g}{4}\,\,\,.
\]

Hence, assuming that Assumption~\ref{assumption h and e} holds for
all splits then the algorithm will split every leaf that contain multiple
big clusters and will not split any leaf that contain a single big
cluster. Therefore, after $K-1$ splits, all the leafs will have a
single big cluster. For each of the $K$ leaf, the DSC algorithm will
call $\alg$ once to determine that it contains a single cluster and
hence the number of calls to $\alg$ will be at most $2K-1$ and in
each call all the clusters are either $\gamma$-big or $\gamma$-small
for $\gamma\leq\gamma_{K}=\nicefrac{4K\epsilon}{g}$. And therefore,
with probability $1-\delta^{*}$, the DSC algorithm will return a
$\gamma_{K}$ clustering tree. Since 
\[
\gamma_{K}=\frac{4\epsilon K}{g}\leq\frac{4K\frac{\epsilon^{*}g}{4K}}{g}=\epsilon^{*}\,\,\,,
\]
this is an $\epsilon^{*}$ clustering tree.\end{proof}
}
\end{document}